\documentclass{article}

\usepackage{microtype}
\usepackage{graphicx}
\usepackage{subcaption}
\usepackage{booktabs} % for professional tables
\usepackage{xurl}

\usepackage{hyperref}

\usepackage[accepted]{icml2026}

\usepackage{amsmath}
\usepackage{amssymb}
\usepackage{mathtools}
\usepackage{amsthm}

\usepackage{tikz}
\usetikzlibrary{shapes.geometric, arrows.meta, positioning, calc, fit, backgrounds}
\usepackage{tikz}
\usepackage{standalone} % Necessary to skip headers of {standalone} documents
\usepackage{tikzscale} % To use \includegraphics with TikZ code
\usepackage{array}
\usepackage{tabularx}
\usepackage{colortbl}
\usepackage{xcolor}

\definecolor{meekblue}{RGB}{31,119,180}
\definecolor{mightyred}{RGB}{214,39,40}
\definecolor{lightmeek}{RGB}{174,214,241}
\definecolor{lightmighty}{RGB}{245,183,177}
\definecolor{decisionbg}{RGB}{245,245,245}

\usepackage{listings}
\usepackage{xcolor}

\lstdefinestyle{pycode}{
  language=Python,
  basicstyle=\ttfamily\footnotesize,
  keywordstyle=\color{blue!55!black},
  commentstyle=\itshape\color{gray},
  stringstyle=\color{green!45!black},
  numbers=left, numberstyle=\tiny\color{gray},
  showstringspaces=false,
  breaklines=true, breakatwhitespace=true,
  columns=fullflexible, keepspaces=true,
  frame=single, rulecolor=\color{black!25},
  xleftmargin=1.5em,
}

\usepackage[capitalize,noabbrev]{cleveref}

\theoremstyle{plain}
\newtheorem{theorem}{Theorem}[section]

\newtheorem{corollary}[theorem]{Corollary}
\theoremstyle{definition}
\newtheorem{definition}[theorem]{Definition}

\theoremstyle{remark}

\newcommand{\R}{\mathbb{R}}

\icmltitlerunning{Two AI Metrics Diverged: Will it Make All the Difference?}

\begin{document}

\twocolumn[
  \icmltitle{Two AI Metrics Diverged: Will it Make All the Difference?}

  \icmlsetsymbol{equal}{*}

  \begin{icmlauthorlist}
    \icmlauthor{Alex Fogelson}{equal,mitft}
    \icmlauthor{Zachary A. Brown}{equal,mitft}
    \icmlauthor{Hans Gundlach}{mitft}
    \icmlauthor{Jayson Lynch}{mitft}
    \icmlauthor{Neil Thompson}{mitft}
  \end{icmlauthorlist}

  \icmlaffiliation{mitft}{MIT FutureTech, CSAIL, Cambridge, MA, USA}

  \icmlcorrespondingauthor{Alex Fogelson}{fogelson@mit.edu}

  \icmlkeywords{Machine Learning, Scaling Laws, Compute, Benchmarks, Time Horizon}

  \vskip 0.3in
]

\printAffiliationsAndNotice{\icmlEqualContribution}

\begin{abstract} As exponential compute scaling continues, will the capabilities of frontier AI models outstrip what is accessible to developers on a small fixed budget? Or will capabilities converge, with ``meek models inheriting the earth"? Building on \citet{gundlach2025meek}, we show that the answer depends on how we value and measure AI capabilities. We discuss conventional performance measures and show that, while validation loss shows a shrinking gap, on other metrics frontier models grow their lead forever. Classifying performance metrics by their functional forms in relation to training (and inference) compute, we provide tight mathematical conditions for determining which metrics favor meek models, and show that bounded performance metrics always do. But careful interpretation of performance metrics is essential: we show that many common bounded metrics have closely-related counterpart metrics that are unbounded (and vice versa). Determining the apt metric in a domain is a prerequisite for policy, since bounded and unbounded metrics may suggest opposing policy responses. If a particular capability --- like software engineering, synthetic biology, or rhetorical persuasiveness --- is unbounded when measured in the terms we care about, frontier-level capability will likely be concentrated in the hands of a few wealthy actors. Conversely, if that capability is instead bounded, frontier-level capabilities proliferate through meek models into the hands of the many. 
\end{abstract}

\section{Introduction}

\subsection{Background}
When AI models are trained with more compute they gain increased capabilities, as measured by (for example) validation loss \citep{kaplan2020scalinglawsneurallanguage,rosenfeld2021scaling, hoffmann2022trainingcomputeoptimallargelanguage,Bahri_2024}. Similarly, when AI models are allowed to reason about problems for longer --- by using more compute at inference time --- capabilities improve, as measured by (for example) success on benchmark tasks \citep{jones2021scaling, epoch2023tradingoffcomputeintrainingandinference}. As a result of these regularities, companies have invested exponentially increasing amounts in compute for pre-training and inference~\citep{Sevilla_2022, epoch2026hyperscalercapextrend}. 

These large expenditures have enabled a small handful of companies to offer more powerful AI capabilities than available elsewhere. 

But is this oligopolistic equilibrium guaranteed? Providers of smaller, cheaper open-weight models have been able to replicate the capabilities of expensive proprietary models with only a short delay --- often less than a year \citep{epoch2024openmodelsreport,emberson2025open_weights_lag}. If the capabilities gap doesn't widen, but instead shrinks, we might expect frontier-level capabilities to diffuse very widely. In such a world, regulating AI capabilities may require compliance from many actors, or require enforcing regulations at the hardware level --- which would require substantial new technological means and political will \citep{ogara2025hardwareenabledmechanismsverifyingresponsible}. 

Given the regulatory relevance, it would be valuable to know whether the capabilities gap between expensive proprietary models and cheaper open models will shrink or widen. However the empirical literature is sparse, ambivalent, and --- as we'll see --- sensitive to one's individual utility function over a range of capabilities \citep{emberson2025open_weights_lag,aiindex2025_ch2_technical_performance,ihle2025weirdml_gap}. 

Against this backdrop, a recent paper \citet{gundlach2025meek} provides a theoretical argument that the gap could shrink. The authors (also authors of this paper) compare the performance of models with exponentially growing investment in training compute to ``meek'' models which are trained at a constant level of compute investment. Both models benefit from exponential improvements in hardware and algorithm efficiency (meek models' \textit{effective} compute grows, even if their \textit{compute investment} doesn't), but meek models ride a slower exponential growth rate. The authors show that, on a few common performance metrics --- validation loss, sigmoidal benchmarks --- the performance of meek models and frontier models converge in the long run. They analyze several specific capabilities, but the long-run properties do not depend on the specific domain, only the functional form of how compute translates into performance. The underlying intuition is that both models are traversing a fundamentally bounded performance metric at different paces as they increase their effective compute: while the frontier model gets near the bound sooner, eventually both models are in the region where the metric doesn't change much (or in some cases, at all) over time. 

\subsection{Capabilities Versus Metrics}
\label{sec:capabilities}

As we will show, this result is sensitive to the specific mathematical properties of the performance metric analyzed. This gives rise to a problem: in machine learning, two metrics used to measure ostensibly similar capabilities can often have very different functional forms with respect to compute. Metrics are often selected for accurately capturing the ordinal ranking of models in some domain, and there is typically no requirement that the cardinal performance level of models capture something especially meaningful. This is an issue if we aim to draw conclusions about a performance gap in units that are meaningful --- if we do, that would require us to use a performance metric that reflects how capabilities improvements are \textit{valued} in some domain. Utility as a function of model capability may be highly non-linear, and the behavior of the selected performance metric needs to take into account this non-linearity if it aims to make meaningful claims about performance gaps or other cardinal properties of the metric.

In this paper, we show that some performance metrics are ``meek'' metrics and some are ``mighty'' metrics. ``Meek'' metrics are ones where ``meek models inherit the earth" --- that is, they eventually see models converge in performance under exponentially diverging compute expenditure. In contrast, ``mighty'' metrics do not converge. We show that if a metric is bounded, then it is meek. \textbf{Critically, we also show that subtle differences in utility function\footnote{A note on terms: by ``utility function'', we are referring to any function that evaluates how much a particular gain in capabilities matters for some real-world aim. Utility functions may differ between actors, tasks, and domains.} over a model's capabilities can produce metrics which change from unbounded to bounded, from mighty to meek, from sensitive to exponential investment gaps to (in the long run) completely indifferent to them. }

Consider the case of two software engineers using coding agents. The first is required to briefly validate the code produced at regular task intervals (e.g. every few human work hours) and aims to improve the accuracy within that interval; the second aims to maximize the length of the task that can be achieved within some reasonable error tolerance. Both engineers prefer models with improving capabilities, but their utility functions produce dramatically different real-world preferences. The first engineer's utility function is \textit{bounded} (since there is a maximum 100\% success rate). Therefore, as time progresses, the first engineer will find less and less difference between models trained with exponentially more compute and models which only leverage shared improvements in algorithms and hardware. By contrast, the second engineer will see continual improvements on their metric from larger and larger models.

\subsection{Contributions}
Upon seeing empirical claims about converging or diverging performance in some domain, it is important to understand that the result may be imposed by the choice of performance metric. The task of an AI analyst, economist, or policymaker is then to determine whether the metric is appropriate.

On the metrics we care about, it is useful to understand the long-term behavior of model capabilities. Our expectation, informed by our discussion of metrics below, is that we will see meek models inherit some parts of the world, and the mighty dominate in other domains. Concentration in some areas; diffusion in others.

The remainder of this paper is as follows. We summarize the contribution of \citet{gundlach2025meek}; define ``meek metrics" and describe the boundary conditions that confer meek metric status (Section~\ref{sec:math}); survey common performance metrics and discuss their sensitivity to interpretation (Section~\ref{sec:wild}); discuss how the framework can be generalized to model gaps in inference compute (Section~\ref{sec:inference}); and conclude by discussing limitations, complexities, and implications for society (Section~\ref{sec:discussion}).

\section{Core Mathematical Definitions and Results}
\label{sec:math}

\subsection{Meek Models Argument}
\label{sec:meekmodels}

Our goal here is to generalize the argument made in \citet{gundlach2025meek}, which analyzed theoretical differences in \textit{validation loss} between so-called \textit{meek models} and \textit{frontier models}. Since validation loss is a power law in effective training compute, one can analytically observe the closed form of the difference as:
\[\Delta L = A(C_{meek})^{-\alpha} - A(C_{frontier})^{-\alpha}\]

In this context, $0<\alpha<1$ and $A$ are parameters governing the power law, and $C$ is \textit{effective training compute}: raw training compute scaling multiplied by shared exponential growth factors from \textit{algorithmic / data progress} and \textit{hardware efficiency} progress. While both frontier and meek models see exponentially growing effective training compute, frontier models have \textit{faster} exponential growth, since they exponentially increase investment in raw training compute over time. 

More concretely, let $g_h$ be the shared annual growth factor of hardware efficiency (in FLOPs per dollar); $g_a$ be the shared annual growth factor of algorithmic and data progress; $g_i$ be the annual growth factor of training compute scaling for the frontier model only; and $C_0$ be some initial effective compute. Then the loss difference yields a difference of two decaying exponentials, which quickly approach zero after a one-time peak.

\[\Delta L  =  A[(g_ag_h)^tC_0]^{-\alpha} - A[(g_ig_ag_h)^tC_0]^{-\alpha}\]

However, as we discuss in Section~\ref{sec:loss}, loss can be an opaque metric for measuring capabilities; it's not obvious how to translate loss into utility. Thus we ask, when does this argument generalize to other performance metrics?

\subsection{Generalized Notion of ``Meek Metrics''}
\label{sec:definitions}
We briefly present a collection of definitions and results which allow for straightforward categorization of performance metrics. Proofs can be found in Appendix~\ref{Sec:Formal Definitions and Proofs}. \\ 

\begin{definition}[Normal Performance Metric]
If a function mapping training compute to performance, $P: \R^+ \rightarrow \R$, is both differentiable and weakly monotonically increasing, we call it a \textbf{normal performance metric}.
\end{definition}
These criteria are natural: (1) Performance growth is typically smooth in compute, and even when performance exhibits sharp jumps, such transitions are still smooth \citep{power2022grokking}, (2) Although adding compute can diminish performance (e.g. overfitting), since we only require weak monotonicity, any model which degrades in performance can simply be discarded in favor of the prior, better model.

Notice that here and throughout when we talk about ``metrics" we are talking not just about a particular way of evaluating performance (e.g. a benchmark) but the curve that describes how \textit{realized or forecasted} performance on that evaluation changes with compute.\\

\begin{definition}[Meek Metric]
Let $P: \R^+ \rightarrow \R$. We say $P$ is a meek metric under exponential scaling if it is a normal performance metric and for all pairs of exponents $b > a > 1$, and for all initial values of training compute $C_0 > 0$, the following limit holds:
$$\lim_{t \rightarrow \infty} P(b^tC_0) - P(a^tC_0) = 0$$
\end{definition}
Intuitively, $P$ is invariant to exponential differences in effective training compute, where $a$ comes from hardware/algorithmic efficiency, and $b$ adds exponential increases in investment. Given some initial training compute investment $C_0$, two actors who see exponential differences in the growth of that effective training compute will see no durable difference in performance in the long run.

\begin{definition}[Mighty Metric]
A metric is mighty if and only if it is a normal performance metric and it is not meek. That is, if $P$ is a normal performance metric but $\lim_{t \rightarrow \infty} P(b^tC_0) - P(a^tC_0) > 0$, then $P$ is mighty.
\end{definition}
 Note that a mighty metric does not \textit{necessarily} exhibit \textit{divergence} between frontier and meek model performance. There could be a constant capabilities gap, for example (see Theorem \ref{thm:meek_classification}).

\subsection{Results}
\label{sec:results}

We now present two results which classify the space of normal performance metrics. % for \textit{bounded} and \textit{unbounded} functions. 
We start with a common and convenient special case of bounded performance metrics.

\begin{theorem}\label{thm:bounded}
If $P$ is a normal performance metric which is bounded above, then $P$ is a meek metric.
\end{theorem}

Although this result follows immediately from the Monotone Convergence Theorem\footnote{The Monotone Convergence Theorem asserts that a monotone real-valued function with an upper bound must converge to its least upper bound.}, it is practically useful given the diversity of performance metrics which are indeed bounded. For example, the original result from \citet{gundlach2025meek} follows as an immediate corollary. \\

For the broader class of unbounded normal performance metrics, we have the following characterization.
\begin{theorem}\label{thm:meek_classification}
Let $P$ be a normal performance metric which is unbounded above. Moreover, suppose the derivative of $P$ with respect to $\log \log C$ is eventually monotone. Then $P$ is a meek metric if and only if 
$$\lim_{C \rightarrow \infty} \frac{P(C)}{\log \log C} = 0$$
\end{theorem}

Intuitively, this bound holds for the following reason: a single application of the logarithm to inputs $b^tC_0$ and $a^tC_0$ still results in two terms which grow distinctly: ($t\log b + \log C_0$) and ($t \log a + \log C_0$). Yet their ratio clearly approaches $\log b / \log a$, such that after applying another logarithm, the limit approaches a constant $\log (\frac{\log b}{\log a}) = \log \log b - \log \log a$. Thus, in the limit, the function $P(C) = \log \log(C)$ results in a constant difference and acts as a sort of boundary point for this convergence.

\textbf{Together, these provide a tight criterion on the space of normal performance metrics}: (1) Boundedness suffices to show meekness; and (2) for the set of unbounded normal metrics which have well-behaved growth, functions are meek if and only if they grow slower than $\log \log C$.

Finally we show that being a meek metric under exponential compute scaling is equivalent to being a meek metric under power-law compute scaling. This shows meekness holds under more general assumptions about the future of progress in computing.

\begin{theorem}[Equivalence Under Power Law Scaling]
\label{thm:power_law_equiv}
Let $P: \R^+ \rightarrow \R$. Then the following are equivalent: 
\begin{enumerate}
\item For all exponents $b > a > 1$ and compute $C_0 > 0$, $\lim_{t \rightarrow \infty} P(b^tC_0) - P(a^t C_0) = 0$
\item For all powers $b > a > 0$ and compute $C_0 > 0$, $\lim_{t \rightarrow \infty} P(t^bC_0) - P(t^a C_0) = 0$ 
\end{enumerate}

\end{theorem}

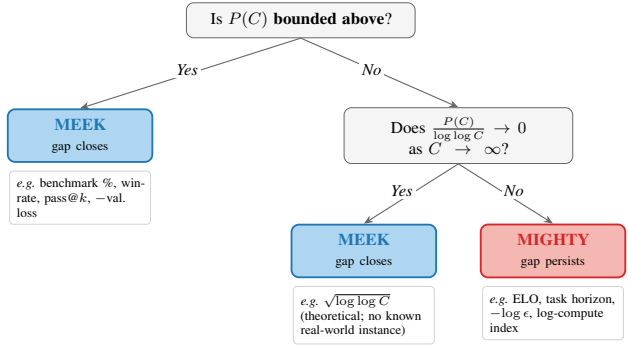
\begin{figure}[t]
\centering
\resizebox{\linewidth}{!}{%
\begin{tikzpicture}[
  every node/.style={font=\footnotesize},
  question/.style={rectangle, rounded corners=4pt, draw=black!60, fill=gray!8,
                   text width=4.0cm, align=center, inner sep=5pt},
  meekbox/.style={rectangle, rounded corners=4pt, draw=meekblue, line width=1pt,
                  fill=lightmeek, text width=2.4cm, align=center, inner sep=5pt},
  mightybox/.style={rectangle, rounded corners=4pt, draw=mightyred, line width=1pt,
                    fill=lightmighty, text width=2.4cm, align=center, inner sep=5pt},
  exbox/.style={rectangle, draw=gray!35, fill=white, rounded corners=2pt,
                text width=2.4cm, align=left, inner sep=4pt, font=\scriptsize},
  arr/.style={-{Stealth[length=5pt,width=4pt]}, semithick, draw=black!60},
  lbl/.style={fill=white, inner sep=1.5pt, font=\footnotesize\itshape},
]

% Root question
\node[question] (q1) at (0,0) {Is $P(C)$ \textbf{bounded above}?};

% Left leaf: bounded → MEEK
\node[meekbox] (m1) at (-4.2,-2.2) {%
  \textcolor{meekblue}{\bfseries MEEK}\\[1pt]
  \normalfont\scriptsize gap closes};

% Second question (unbounded branch); positioned so its center is right of q1
\node[question] (q2) at (3.0,-2.2) {Does $\frac{P(C)}{\log\log C} \!\to\! 0$ as $C \rightarrow \infty$?};

% Second-level leaves
\node[meekbox] (m2) at (1.2,-4.4) {%
  \textcolor{meekblue}{\bfseries MEEK}\\[1pt]
  \normalfont\scriptsize gap closes};
\node[mightybox] (mg) at (4.8,-4.4) {%
  \textcolor{mightyred}{\bfseries MIGHTY}\\[1pt]
  \normalfont\scriptsize gap persists};

% Yes arrow: diagonal left to m1
\draw[arr] (q1.south) -- node[lbl, left, pos=0.45] {Yes} (m1.north);

% No arrow: L-shaped so it arrives vertically and centered on q2.north
\draw[arr] (q1.south) -- node[lbl,  pos=0.45] {No} (q2.north);

% Second-level arrows
\draw[arr] (q2.south) -- node[lbl, left,  pos=0.45] {Yes} (m2.north);
\draw[arr] (q2.south) -- node[lbl, right, pos=0.45] {No}  (mg.north);

% Example boxes
\node[exbox, below=0.15cm of m1] {%
  \textit{e.g.}\ benchmark \%,
  win-rate, pass@$k$,
  $-$val.\ loss};
\node[exbox, below=0.15cm of m2] {%
  \textit{e.g.}\ $\sqrt{\log\log C}$\\
  (theoretical; no known
  real-world instance)};
\node[exbox, below=0.15cm of mg] {%
  \textit{e.g.}\ ELO, task
  horizon, $-\!\log\epsilon$,
  log-compute index};

\end{tikzpicture}
}
\caption{\textbf{Metric classification decision tree.}
Any normal performance metric is \textcolor{meekblue}{\textbf{meek}} (gap converges) or \textcolor{mightyred}{\textbf{mighty}} (gap persists) and can be classified with two questions.
Bounded metrics are always meek (Theorem~\ref{thm:bounded}); for unbounded metrics meekness holds iff the metric grows slower than $\log\log C$, subject to the weak growth conditions given in Theorem~\ref{thm:meek_classification}.}
\label{fig:decision-tree}
\end{figure}

\section{The Subtleties of Performance Metrics in the Wild}
\label{sec:wild}

The meek metric criterion is a strong one: if a performance metric is meek, it will eventually be indifferent to exponential differences in compute investment. Naturally, one might wonder which metrics, if any, are meek metrics. This question turns out to be both subtle and consequential. At first glance, the literature consists mostly of two types of metrics: bounded metrics which are meek (e.g. benchmarks, negative validation loss, inference scaling) and non-meek power law metrics (e.g. game performance, time horizon). Yet upon further examination, many meek (or non-meek) metrics have closely-related alternative metrics which are non-meek (or meek). This has important implications for those tracking and governing AI progress: two slightly different interpretations of the same capability imply profoundly different relationships to compute expenditure. 

We now present examples from the literature of common performance metrics to highlight these subtleties. In addition, we hope to supply the reader with a concise overview of many common metrics and their relationship to compute.

\subsection{Meek Metrics}
\subsubsection{Validation Loss}
\label{sec:loss}

The most notable example of a meek performance metric is, of course, validation loss predicted by neural scaling laws \citep{kaplan2020scalinglawsneurallanguage,rosenfeld2021scaling, hoffmann2022trainingcomputeoptimallargelanguage}. However, validation loss is a difficult metric to use on its own since (1) its interpretation is information-theoretic, rather than directly describing performance on a real-world task and (2) it \textit{diminishes} to zero. To resolve the latter trouble, one can apply transformations to loss to make it a Normal Performance Metric, one that is unbounded and increasing, but this choice has inherent freedom. As we discuss in Section~\ref{sec:bounded}, negative log-loss is a common and useful transformation in predicting other metrics (e.g. benchmarks).

\subsubsection{Benchmarks}
\label{sec:bounded}

In the era of large language models (LLMs), bounded metrics are by far the most common meek performance metric due to the prevalence of benchmarks scored from $0$-$100$. Even pre-dating LLMs, benchmarks like ImageNet \citep{deng2009imagenet}, CIFAR \citep{krizhevsky2009learning}, and MNIST \citep{lecun1998gradient} were foundational for measuring progress across computer vision.

Benchmarks which are bounded between $0$ and $100$ are meek precisely due to their boundedness. As analyzed originally in \citep{gundlach2025meek}, we should expect benchmark scores to converge across exponentially diverging compute investment as progress in hardware and algorithms / data brings even fixed-compute models closer to 100\%.

A body of existing work relating benchmarks, loss, and training compute demonstrates that benchmarks appear to be sigmoidal in negative log-loss (or log-compute\footnote{Under compute optimality, neural scaling laws derive a power law between compute and loss. Therefore, one expects log-loss and log-compute to be affine.}). That is: $\text{BenchmarkAccuracy} = \sigma(-\log L) = \sigma(\log C)$. \citet{owen2024predictable} fits a logistic sigmoid to Big-Bench \citep{srivastava2022beyond} and MMLU \citep{hendrycks2020measuring} scores as a function of compute. More generally, \citet{ruan2024observational} computes the principal components of (logit-transformed) benchmarks and finds correlations between the primary component and log-compute. Both papers find this sigmoidal relationship between performance and log-compute. Though one might naturally expect sigmoidal shapes for metrics which are bounded on both sides, why is the right transformation of compute (or loss) \textit{logarithmic?}

Extending \citet{schaeffer2023mirage}, we posit an explanation which not only justifies the presence of the logarithm, but also argues that benchmarks can be naturally thought of as a transformation of $\log(1/\epsilon)$ where $\epsilon$ is a per-token, amortized, task level error rate, and $\log(1/\epsilon)$  can be thought of as error rate orders of magnitude or ``the number of nines of reliability."\footnote{Technically, it's proportional to the number of nines, since we use natural log throughout.} Importantly, unlike the benchmarks it explains, $\log(1/\epsilon)$ is a mighty metric (the so-called ``march of nines"). 

Schaeffer first notes that validation loss is roughly the (negative) log-probability of correctly producing a particular token, $L \approx -\log p_i$ (when irreducible loss is negligible). In this case, $L$ is the reducible loss --- total cross-entropy minus the irreducible floor. For some number of tokens $T$ required for a single question, idealizing these draws as independent, one arrives at the following expression for accuracy as a function of compute, for some scaling exponent $0 < \alpha < 1$:

\vspace{-.25cm}
$$\text{Accuracy}(C) = {\underbrace{(e^{\log p_i})}_{\text{success rate}}}^T = (e^{-L})^T = e^{(-AC^{-\alpha} T)}$$
\vspace{-.25cm}

This function is a very gradual sigmoid in $C$, but as a function of $\log C$, it becomes the Gompertz function, a more abrupt sigmoid (explaining why benchmarks are empirically well modeled as sigmoidal in log compute). Writing the error rate using the approximation $e^{-x} \approx 1-x$ for small $x$, we get
$$\epsilon(C) = 1-\text{Accuracy}(C) = 1-e^{-LT} \approx LT$$

From this one straightforwardly sees that halving the error rate requires halving the loss, which itself requires a $2^{1/\alpha}\times$
increase in compute (for Chinchilla scaling roughly a 90× increase in training compute would be required to halve one’s error \citep{hoffmann2022trainingcomputeoptimallargelanguage}). Equivalently, $\log(1/\epsilon)$ is linear in log compute:
$$\log(1/\epsilon) \approx -\log(LT) \approx \alpha \log{C} - \log(AT)$$

This connection to benchmarks is consistent with other work from \citet{ho2025rosetta}, which finds the fundamental capability scale is linear in log-compute. Using a suite of benchmark scores, and assuming a sigmoidal form relating some latent model capability to downstream benchmark score, Ho et al. use an item-response theory (IRT) model to back out a universal model capability index. That index correlates strongly with log-compute (and thus likely log-loss), suggesting that the right scale for measuring capabilities is logarithmic in compute (or loss). 

Both of these expositions of benchmarks suggest benchmark performance may be best explained by a fundamental scale which is linear in log-compute -- and therefore a \textit{mighty metric} -- despite the benchmark itself being a \textit{meek metric}. Which scale is correct depends on one's relationship to the underlying measure: if one's utility is in error rate orders of magnitude, meek and frontier models will diverge. If near perfect accuracy on a fixed benchmark suffices, meek models will indeed inherit the earth.

\subsubsection{Misalignment}
This phenomenon can also arise when the functional form of benchmarks does not reflect the true utility of the underlying capability. For example, suppose one is testing a model's misalignment through a fixed benchmark \citep{zhang2023safetybench, mazeika2024harmbench, gabor2025evilgenie}. We \textit{could} measure misalignment on a suite of benchmarks with scores between 0 and 100\%. But suppose that the harm caused by an instance of misaligned behavior increases with time $t$ --- since ever-more capable models can cause bigger harms \cite{anthropic2026mythos} --- and so the size of a harm scales by (for example) $r^t$. Utility (or rather, disutility) might be better described by the \textit{expected harm}, rather than the frequency of harm, and thus $H(t, \epsilon) = \epsilon r^t$ may be a more suitable measure of disutility than $\epsilon$ alone (where $\epsilon$ is the proportion of the time the model displays misaligned behavior on the benchmark tasks). This metric grows unboundedly for sufficiently large $r$ relative to the rate of effective compute growth. This example again illustrates a broader pattern: benchmark accuracy is bounded and therefore meek, but benchmark accuracy often has closely related transformations with natural interpretations that are unbounded, and therefore possibly mighty.

\subsection{Non-Meek Metrics}
\label{sec:power law}

We first investigate two examples of power law performance metrics, where multiplicative increases in compute result in (usually smaller) multiplicative increases in performance. In contrast to the power laws in neural scaling laws, these metrics are monotonically \textit{increasing} in compute. We also discuss inference time scaling --- which is linear in log-compute --- requiring multiplicative changes in compute for linear changes in performance.

\subsubsection{Reinforcement Learning in Games}
\label{sec:games}

Game environments have been fundamental in the development of reinforcement learning and deep learning more generally \citep{bellemare2013arcade,mnih2015human}. A robust literature on compute scaling gives ample examples across domains, though typically using an ELO score. Notably, as \citet{neumann2022scaling} points out, ELO is merely the Bradley-Terry strength on a logarithmic scale, where two Bradley-Terry scores $\gamma_i$ and $\gamma_j$ correspond to a win-rate for player $i$ of $\gamma_i/(\gamma_i + \gamma_j)$.\footnote{Bradley-Terry has a natural interpretation: each player gets $\gamma_i$ lottery tickets (equal to their Bradley-Terry strength), and a single random draw determines the winner of the lottery. Thanks to Toby Ord for mentioning this interpretation and related discussions.} Thus any relationship which relates ELO as log-linear in compute, in fact, yields a power law in Bradley-Terry strength.

Indeed this is exactly what is found in the literature. In Hex, Pentago, and Connect Four, \citet{jones2021scaling} and \citet{neumann2022scaling} show exponential fits between ELO scores and training compute. This relationship is also observed by \citet{thompson2022importance} for ELO in historical Chess and Go systems for a mixture of machine learning and classical AI systems.

Performance in deterministic games with unknown optimal strategies may be unbounded, depending on the game and the possibility of draws. In a game without draws, for example, the Bradley-Terry strength $\gamma_i$ could be increased arbitrarily against some fixed reference opponent with strength $\gamma_j$. This would make Bradley-Terry an unbounded metric, even though the win-rate ($\gamma_i/(\gamma_i + \gamma_j)$) exists on a separate (bounded) scale. Bradley-Terry scores are mighty, while win-rates are meek.

This fact is quite intuitive. Consider chess engines: although chess engines may be able to improve indefinitely with exponential compute (or at least for many, many orders of magnitude), engines beyond a certain strength already have what they need to best any human. Though it was once surprising to see a computer beat grandmaster Garry Kasparov, improvements in hardware and algorithms have rendered any budget smartphone equally capable of defeating world-champions. With respect to chess win-rates, meek models have inherited the earth, a fact that now seems unsurprising.

\subsubsection{Task Horizon Length}
\label{sec:horizon}
Notably deployed by \citep{kwa2026measuring}, the task horizon length of a model measures the difficulty of reference tasks by how long completion takes humans, on average. Across a range of tasks of varying duration, one first fits a curve to predict the model's probability of success. After deciding on some threshold success probability (typically 50\%), one can derive the expected task length that the model will complete. Over time, the expected duration at a 50\% success rate appears to be increasing exponentially.

Analyses of the relationship between task horizon length and training compute show power-law relationships (\citet{whitfill2025forecasting}, although some results are imputed from benchmark scores). Related theoretical work gives a simple mechanism for this pattern \citep{sinha2026illusion}: if a task of length $n$ requires $n$ independent steps each to succeed with probability $p=1-\epsilon$, then achieving fixed task success probability $q$ requires
\[
n=\frac{\log(q)}{\log(1-\epsilon)}\approx \frac{-\log(q)}{\epsilon}.
\]
Hence fixed-threshold task horizon grows approximately in proportion to the inverse single-step error rate. As discussed above, error typically falls as a power law in compute, meaning it falls exponentially in time under exponential compute growth. So task horizon rises exponentially in time. 

Here again we find that an unbounded mighty metric --- task horizon length --- is a transformation of a bounded meek metric --- the per task error rate. Depending on circumstances, either metric may be the utility-relevant one.

\section{Extending to Inference Time Scaling}
\label{sec:inference}

We've defined meekness with respect to gaps in effective \textit{training} compute. However, the Meek Models Framework extends naturally to gaps in \textit{inference-time computation} budgets. Suppose two models have equal effective training compute, but one has an inference token budget growing at a quick exponential rate, as a result of an exponentially growing dollar budget on top of inference efficiency improvements and hardware efficiency improvements. Meanwhile, the other has a token budget growing at a slower exponential rate, getting the benefit only of inference and hardware efficiency improvements. Will the performance of the latter model catch up over time? 

The existing literature on inference scaling laws mostly uses benchmark accuracy as a metric, which (as we've seen) is inherently bounded and therefore meek, regardless of how fast benchmark accuracy rises. In practice, the particular functional forms for benchmark accuracy with respect to inference compute vary, depending on the inference-scaling technique used \citep{villalobos2023infscaling}. But in general, a common empirical pattern is that performance improves predictably with additional inference compute over a substantial range before eventually exhibiting diminishing returns \citep{brown2024large,ellismohr2025theory}. 

For example, take the particular inference scaling technique of repeated sampling with verification. Benchmark performance using this technique is often well fit by an exponentiated power law,
\[
\log\left(\mathrm{pass}_i@k\right) \approx a k^{b},
\]
where $k$ is the number of attempts (i.e. the amount of inference compute) \citep{brown2024large} and $a,b<0$ are fitted constants. This is the same functional form as we  derived above for benchmark performance as a function of log training compute, just this time in terms of log inference compute.\footnote{Interestingly, one paper argues that the reason we see this functional form in particular depends not just on the fundamental dynamics of test-time compute scaling, but also on the distribution of question difficulty across multi-question benchmarks \cite{schaeffer2025large}. (If all questions were equally difficult, performance would still scale to a bound, but you'd expect \textit{exponentially} decaying performance.) This is another reason to be careful when interpreting benchmarks' scaling trajectories, whether with respect to inference or training compute --- especially since the difficulty of questions across benchmarks is not something that benchmark designers typically design deliberately. } Since the pass rate is bounded above by $1$, it is a meek metric under our definition: even exponential differences in inference expenditure do not generate a permanent gap in this metric. (As we saw with training compute, though, the corresponding ``march of nines" metric is mighty.)

\citet{jones2021scaling} examines ELO vs perfect play in various games, where ELO eventually plateaus in inference compute. Note that ELO isn't constructed in such a way that makes unbounded performance theoretically impossible: inference compute just doesn't empirically yield continually increasing ELO performance, and so ELO-with-respect-to-\textit{inference}-compute turns out to be a meek metric. Finally, we see similar behavior so far for performance on METR's time horizon, where we see sharply diminishing returns with inference scaling \citep{ord2025hourlycosts}, even though the metric is theoretically unbounded.

The preceding discussion in this section has considered two models with equal effective training compute, but different inference compute budgets \textit{on a single task.} Two related but different questions are also worth considering:
\begin{enumerate}
    \item How do performance gaps on a task change when one model has a growing training compute budget \textit{and} a growing inference compute budget?
    \item How does the gap in \textit{economic returns} change when one actor has an exponentially growing inference compute budget, and uses it to do more economic tasks while holding the amount of inference per task fixed? 
\end{enumerate}

We leave both of these questions as avenues for future work: the first because more work is needed to define a joint training-inference scaling law, the second because this is a fundamentally economic question, beyond our scope here, and the answer will likely vary a great deal by domain.

\section{Discussion}
\label{sec:discussion}

\subsection{Limitations}
\label{sec:Limitations}

\subsubsection{Positional metrics}
\label{sec:positional}
In some domains, rewards are positional, where ordinal rankings of performance matter but cardinal positions don't. An extreme case of this is winner-takes-all competition, such as: a government contract that goes to the model that performs best on some metric. Here, a real-world reward is distributed on the basis of the ordinal position of models on an underlying metric. Positional metrics can't be modeled with the ``meek" framework. This is because these metrics are functions that take in \textit{multiple} players' capabilities. They are also often non-differentiable in compute: small compute increases can cause a player to leapfrog a competitor. 

However, the transformation of raw capabilities into positional metrics is in key respects similar to the transformations of bounded metrics. Just as many of the meek metrics we describe have related metrics that are non-meek that we might care about (and vice versa), many capabilities that have meek metrics have positional metrics where rewards always go to frontier models. Consider: even when an underlying performance metric is asymptotically bounded, the positional reward might always go to a frontier model with more compute. This is because, though all models may converge to equivalent performance in the limit, at any particular $t$, the frontier model has some (possibly infinitesimal) advantage over the meek model.

That said, this case can be subtle, since often what matters is the \textit{expected} positional reward ex ante when there is some stochasticity in how performance translates to position. This can be modeled with win-rates or ELO, described above.

\subsubsection{Alternative forms of compute scaling}
In our discussion of benchmarks, we've assumed that frontier model builders scale investment in compute exponentially. A reader may wonder if this assumption is reasonable, and if it is consequential for our conclusions.

\paragraph{Should we expect an exponentially growing compute gap?} Historically, frontier model compute scaling \textit{has} been exponential \citep{EpochAIModels2025}. However, the historical pattern may not persist. Persistence will require revenue to grow sufficiently quickly with compute scale to justify the expenditure. So far, it has \citep{epoch2025openairevenue}. But if, at some point, the revenue companies obtain from scaling is bounded, or does not grow sufficiently quickly, the \textit{gap in compute} between frontier models and meek models may itself shrink over time. As a result, \textit{revenue generated by AI systems} is a singularly important metric of AI capabilities. We welcome more efforts to understand how revenue scales with compute investment.

\paragraph{If compute gaps grow sub-exponentially, would this change our findings?} Theorem~\ref{thm:power_law_equiv} says that power-law scaling of compute would not overturn our assessment of meekness in any case. That said, convergence would be delayed. And substantially slower-than-polynomial scaling could have different results, especially for unbounded metrics. 

\begin{figure}[tbp]
    \centering
    \includegraphics[width=1\linewidth]{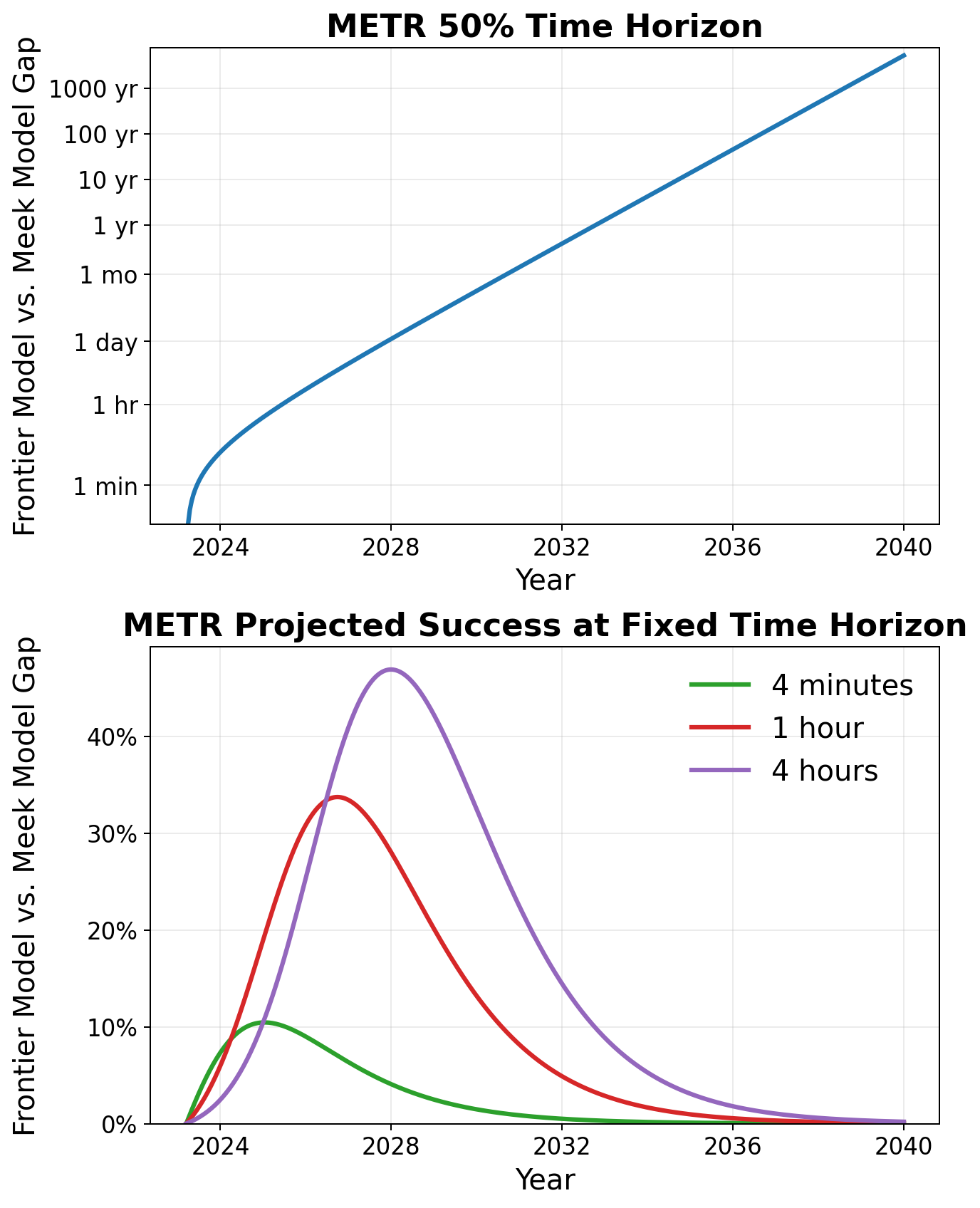}
    \caption{Gap between frontier model performance improvements and the imputed meek model performance (following the same trend more slowly, coinciding with the frontier at GPT-4) on two related metrics: the METR time horizon \citep{kwa2026measuring} (Top) and the probability of task success at 4-minute, 1-hour, and 4-hour tasks (Bottom). Both metrics are generated from the same underlying trend, but some lead to a meek outcome and some do not. These results are stylized estimates, not forecasts.}
    \label{fig:METR_metrics}
\end{figure}

\subsubsection{Alternative forms of algorithmic and data progress}
\paragraph{Scale-dependent algorithmic and data progress.} For simplicity, we've modeled algorithmic and data progress with the single coefficient $g_a$. This functions like a constant multiplier on a model's pretraining compute. Recent work \citep{gundlach2025originalgorithmicprogressai} finds that much empirical algorithmic progress is scale-dependent, increasing the scaling law \textit{exponent}, and thereby implying an increasing multiplier on models with more total compute. Incorporating exponent-shift algorithmic progress of this sort in our framework would not change our results, since this is equivalent to multiplying $g_i$ by some factor.

\paragraph{Proprietary algorithmic and data progress.} We model hardware and algorithmic/data efficiency increases as benefiting all AI developers; this is a simplification. A frontier AI company may be able to generate hardware, algorithmic or data innovations that they prevent from diffusing to the rest of the industry \cite{thompson2026is} due to proprietary data, algorithmic innovations that are kept secret, or innovations that are specific to the firm's combinations of models, hardware, and agentic scaffolding \cite{gundlach2026wrapper}. 

For this reason, rates of hardware and algorithmic progress for meek models could conceivably slow substantially, attenuating the core mechanism by which meek models catch up to the frontier. However, as long as the rate of shared, field-wide algorithmic progress doesn't go \textit{all} the way to zero, the meek framework still applies in the long run. 

\subsubsection{Near-term Predictions.} This paper addresses whether metrics will show convergent capabilities in the long term. But even if capabilities converge in the limit, frontier models may pull ahead in the short-term. All meek metrics (which are non-zero) have gaps that first rise and then fall (some unusual ones may have multiple peaks). Convergence may take a long time, yielding a substantial window where the regulatory landscape looks more like a ``mighty'' world. See Fig. \ref{fig:METR_metrics} for an example of what this can look like in practice, using multiple metrics drawn from METR analysis \cite{kwa2026measuring} (details in \ref{app:plot_details}).

\subsection{Governance and Implications}
\label{sec:governance}

% -- implications for incentives to invest
\paragraph{Incentives to invest.} Whether a metric is meek has direct bearing on the private return to compute scaling. For meek metrics, exponential investment yields only a transient advantage: a frontier developer who outspends competitors by orders of magnitude can expect that lead to erode as hardware and algorithmic progress lift the meek competitor's compute budget along the same curve. For mighty metrics, by contrast, the return to scaling is durable, and investment is a moat. The commercial case for continued exponential compute growth therefore depends on which metrics frontier firms can monetize and what the returns on improved performance are --- a question that is itself contested and may differ across domains.

% -- implications for concentration of power
\paragraph{Concentration of power.} When unbounded performance is valuable, owners of compute capital can entrench their advantage: firms or nations who can sustain exponential compute expenditure can consistently control capabilities that smaller players can't access. Those actors may have the ability to dictate terms of use, restrict harmful or privately disadvantageous applications of frontier capabilities, and leverage their capabilities to capture private rents. When the performance metrics we care about are bounded, we see the opposite dynamic: a broadened set of actors who can operate at the frontier in the long run.

The overall future is likely mixed: some sectors concentrated, others proliferated.

% -- national AI advantages and compute controls
\paragraph{Compute controls and national advantages.} Export controls on advanced compute \cite{ogara2025hardwareenabledmechanismsverifyingresponsible} presume that restricting a rival's effective compute will durably restrict its capabilities. If the capabilities in question are ones where unbounded performance matters, compute controls can indeed sustain a durable gap. But for capabilities where performance is bounded when measured in the terms we care about --- that is, cases where meek metrics are apt --- the presumption that compute restrictions lead to persistent capabilities gaps fails in the long run. Provided the restricted actor still enjoys some exponential rate of hardware and algorithmic progress (even if it's a slower exponential rate than the frontier actor enjoys), its capabilities on meek metrics will eventually catch up. \textbf{Compute controls then function as a delay rather than a permanent ceiling on an adversary's capabilities.} 

% -- diffusion of dangerous capabilities 
\paragraph{Dangerous capabilities.} 
For some dangerous capabilities, the relevant social harm may be largely realized once a model crosses a fixed capability threshold: for example, reliably helping a user reproduce a software vulnerability, automate a phishing campaign, synthesize a dangerous protocol from dispersed biological information, or guide a novice through key steps in a CBRN-relevant workflow \citep{frontiermodelforum2025capabilityassessments, mouton2024operational, openai2024bioriskearlywarning, peppin2024reality, zhang2025cybergym, anthropic2025biorisk}. \textbf{In these cases, marginal improvements above the threshold may matter much less than whether the threshold is crossed at all.} If hardware and algorithmic progress continue to raise the effective compute available to fixed-budget developers, then the meek-models result has a direct governance implication: even if only frontier developers can access the dangerous capability today, many meek model builders may eventually access substantially the same threshold capability. If \textit{any} such capabilities impose sufficiently high existential risks, this may be unacceptable \citep{jones2024dilemma}.

This does not mean that every dangerous capability is best understood as meek. Some misuse-relevant quantities may be unbounded or positional: the number of targets that can be attacked, the speed of exploitation, the ability to adapt to active defenders. In particular, societal safety often depends on an offense-defense balance.\footnote{Offense and defense need not be measured on the same performance scale. A model that marginally improves vulnerability discovery may help attackers, while a different model that improves patch generation, intrusion detection, or incident response may help defenders; the ultimate outcome depends on how these two \textit{different} types of performance interact. The meek framework can help ask whether a particular underlying capability --- like cyber knowledge --- will diffuse to meek models, but it does not by itself determine whether diffusion favors attackers or defenders.}

% -- alignment
\paragraph{Alignment.} If meek models inherit the earth, the future is shaped not by one or a few aligned frontier systems but by a population of many roughly frontier-capable models. Which statistic of that population matters---average alignment, maximum alignment, or minimum alignment among any accessible system---depends on the threat model. On pessimistic views in which a single misaligned frontier system suffices for catastrophe, proliferation of frontier-level capabilities is alarming \cite{hammond2025multiagentrisksadvancedai,Bostrom2019Vulnerable}. On more optimistic views in which defense aggregates across aligned systems, and average alignment is what matters, proliferation may be protective. 

\subsubsection{A call for utility-aware analysis.} Empirical claims about convergence or divergence in AI capabilities cannot be naively read off a single metric. The meekness of a metric is a property of the construction of the metric, and closely-related metrics can differ in meekness despite measuring ostensibly similar phenomena. For example, on one metric, we might see the gap in capabilities between the US and China growing, and on another similar metric see it shrinking; which metric is ``right" is a question about how utility changes with respect to a measured increase in capabilities, but metrics are often designed without this question in mind. Analysts, economists, and policymakers should therefore use caution when using existing metrics to draw conclusions about how capabilities gaps are changing. And they should treat the construction and selection of a performance metric as a substantive decision, one that necessarily implies a point of view on how AI capabilities matter. Finally: a better understanding of which metrics most faithfully capture capabilities of economic and social value --- and how models scale on those metrics --- is essential for long-term AI policy.

\section*{Impact Statement}

The authors believe AI capabilities are of fundamental importance to effective governance, geopolitical strategy, scientific progress, economic prosperity, and global welfare. The current understanding of these capabilities relies heavily on the metrics discussed in this paper. Therefore, we see careful analysis of these metrics --- and their relation to compute and capital --- as essential for informing public discourse, improving national security, and guiding policy decisions. 

\section*{LLM Usage Statement}

The authors of this paper used LLMs for generating graphics, styling the text, validating the mathematical results, guidance during mathematical derivations, and literature review.

\bibliography{references}
\bibliographystyle{icml2026}

%%%%%%%%%%%%%%%%%%%%%%%%%%%%%%%%%%%%%%%%%%%%%%%%%%%%%%%%%%%%%%%%%%%%%%%%%%%%%%%
%%%%%%%%%%%%%%%%%%%%%%%%%%%%%%%%%%%%%%%%%%%%%%%%%%%%%%%%%%%%%%%%%%%%%%%%%%%%%%%
% APPENDIX
%%%%%%%%%%%%%%%%%%%%%%%%%%%%%%%%%%%%%%%%%%%%%%%%%%%%%%%%%%%%%%%%%%%%%%%%%%%%%%%
%%%%%%%%%%%%%%%%%%%%%%%%%%%%%%%%%%%%%%%%%%%%%%%%%%%%%%%%%%%%%%%%%%%%%%%%%%%%%%%
\newpage
\appendix
\onecolumn

\section{Formal Definitions and Proofs}
\label{Sec:Formal Definitions and Proofs}

\begin{definition}[Normal Performance Metric]
If $P: \R^+ \rightarrow \R$ is both differentiable and weakly monotonically increasing, we call it a \textbf{normal performance metric}.
\end{definition}
These criteria are natural: (1) Performance growth is typically smooth, and even when performance exhibits sharp jumps, such transitions are still smooth \citep{power2022grokking}, (2) although adding compute can diminish performance (e.g. overfitting), since we only require weak monotonicity, any model which degrades in performance can simply be discarded.\\

\begin{definition}[Meek Metric]
Let $P: \R^+ \rightarrow \R$. We say $P$ is a meek metric if for all pairs of exponents $b > a > 1$, and for all initial values of compute $C_0 > 0$, the following limit holds:
$$\lim_{t \rightarrow \infty} P(b^tC_0) - P(a^tC_0) = 0$$
\end{definition}
Intuitively, $P$ is invariant to exponential differences in effective compute, where $a$ comes from hardware/algorithmic efficiency, and $b$ adds exponential increases in investment. Given some initial compute investment $C_0$, two actors who see exponential differences in the growth of that effective compute will see no durable difference in performance over time.

\begin{theorem}[Equivalence Under Power Law Scaling]
Let $P: \R^+ \rightarrow \R$. Then the following are equivalent: 
\begin{enumerate}
\item For all exponents $b > a > 1$ and compute $C_0 > 0$, $\lim_{t \rightarrow \infty} P(b^tC_0) - P(a^t C_0) = 0$
\item For all powers $b > a > 0$ and compute $C_0 > 0$, $\lim_{t \rightarrow \infty} P(t^bC_0) - P(t^a C_0) = 0$ 
\end{enumerate}

\end{theorem}

\begin{proof}
The proof is an exercise in reparameterization. 

To show $(1)\implies(2)$, let $s = \log t$. (Where $\log$ is the natural log throughout this appendix.) For some powers $b > a > 0$, we have that
\begin{align*}
P(t^b C_0) - P(t^aC_0) &= P((e^b)^sC_0) - P((e^a)^s C_0)
\end{align*}
Since $e^b > e^a > 1$, the limit as $t \rightarrow \infty$ (and equivalently $s \rightarrow \infty$) is zero by $(1)$. To show $(2)\implies(1)$, for some exponents $b > a > 1$, instead let $s = e^t$, so we have
\begin{align*}
P(b^tC_0) - P(a^tC_0) &= P((e^{\log b})^tC_0) - P((e^{\log a})^tC_0) \\
&= P(s^{\log b}C_0) - P(s^{\log a} C_0)
\end{align*}
Since $\log b > \log a > 0$, the limit as $s \rightarrow \infty$ (and equivalently $t \rightarrow \infty$) is zero by $(2)$.

\end{proof}

\begin{theorem} If $P$ is a normal performance metric which is \textit{bounded above}, then $P$ is a meek metric.
\end{theorem}

\begin{proof}
Let $P$ be a normal performance metric with least upper bound $M$. Let $f_x(t) = P(x^tC_0)$. Since $b^tC_0$ and $a^tC_0$ are monotonically increasing in $t$, as is $P$, the functions $f_a(t)$ and $f_b(t)$ are monotonically increasing and bounded above by $M$, and therefore both converge to $M$ as $t \rightarrow \infty$ by the Monotone Convergence Theorem.
\end{proof}

\begin{theorem}
Let $P(C)$ be a monotonically increasing, differentiable, unbounded function. Moreover, suppose the derivative of $P$ with respect to $\log\log{C}$ is eventually monotone. Then $P$ is a meek metric if and only if $$\lim_{C \rightarrow \infty} \frac{P(C)}{\log \log{C}} = 0$$
\end{theorem}

\begin{proof}

We'll show two results, which together give the result:
\begin{enumerate}
    \item Let $D(C)$ be the derivative of $P(C)$ with respect to $\log\log C$. Then $$\frac{P(C)}{\log \log C} \rightarrow 0 \iff D(C) \rightarrow 0$$
    \item For some arbitrary exponents $b > a > 1$ and some initial compute $C_0$, let $\Delta(x) = P(b^tC_0) - P(a^tC_0)$  --- the relevant difference in the meek metric criteria. If $D(x)$ is eventually \textit{increasing}, then for some function $I$ such that $\lim_{x \rightarrow \infty} I(x) = \log \lambda$, we have the following bounds. $$D(x) I(x) \leq \Delta(x) \leq D(kx^\lambda) I(x)$$
    where $\lambda$ and $k$ are non-zero constants which depend on $b$, $a$, and $C_0$. If $D(x)$ is eventually \textit{decreasing}, the inequalities are flipped.
\end{enumerate}

 From (2), we can see that $\Delta(x)$ goes to zero if and only if $D(x) \rightarrow 0$ since $\log \lambda$ is non-zero, and therefore if and only if $P(C)/\log \log{C} \rightarrow 0$ by (1).\\\\

\textbf{Proof of Result 1:} First note that $P(C) \rightarrow \infty$ (by assumption) and $\log \log C \rightarrow \infty$ as $C \rightarrow \infty$. Since $D(C)$ is monotonic, the extended limit certainly exists. If the limit is finite, we can invoke L'Hopital's rule to show that $P(C)/\log \log C$ and $D(C)$ share a limit:
\begin{align*}
=& \lim_{ C \rightarrow \infty} \frac{P(C)}{\log \log C} \\
\overset{\text{LH}}{=}& \lim_{C \rightarrow \infty} \frac{dP/dC}{d(\log \log C)/dC}  \\
=& \lim_{C \rightarrow \infty} \frac{dP}{dC} \frac{dC}{d(\log \log C)} \\
=& \lim_{C \rightarrow \infty} \frac{dP}{d(\log \log C)} \\
=&\lim_{C \rightarrow \infty} D(C)
\end{align*}
If the limit of $D(C)$ is infinite, we instead show $P(C)/\log\log C$ goes to infinity. First, for all $M > 0$, we know $D(C) \geq M$ for large enough $C$, and thus $(dP/dC) \geq M/(C\log C)$. Choosing some small fixed $c_0$, we can integrate $dP/dC$ to get:
$$P(C) - P(c_0) = \int_{c_0}^C \frac{dP}{du} du \geq \int_{c_0}^C \frac{M}{u \log u} du = M(\log \log C - \log \log c_0)$$
Rearranging, we get $$\frac{P(C)}{\log \log C} \geq M + \frac{\mathcal{O}(1)}{\log\log C}$$
So clearly $\liminf_{C \rightarrow \infty} P(C)/\log \log C \geq M$, and since $M$ is arbitrary, it is unbounded like $D(C)$.

\textbf{Proof of Result 2:} We first re-parameterize the second argument in the meek metric criteria, $a^tC_0$, as $x$. Then we define two constants $\lambda = \log_ab > 1$ and $k = C_0^{1-\lambda}$, such that we express $b^tC_0$ as $kx^\lambda$. We re-write the meek metric difference, $\Delta(x)$, using these parameters:

\[\Delta(x) = P(kx^\lambda) - P(x) \]

To bound $\Delta(x)$, we write it in its integral form, using the relation $D(C) = C \log C \frac{dP(C)}{dC}$:

\begin{align*}
    & \lim_{x \rightarrow \infty} \Delta(x) \\
    =&\lim_{x \rightarrow \infty} P(kx^\lambda) - P(x) \\
    \overset{*}{=}& \lim_{x \rightarrow \infty}\int_x^{kx^\lambda} \frac{dP}{du} du \\
    =& \lim_{x \rightarrow \infty}\int_x^{kx^\lambda} \frac{(u\log{u})\frac{dP}{du}}{(u\log{u})} du\\
    =& \lim_{x \rightarrow \infty}\int_x^{kx^\lambda} \frac{D(u)}{(u\log{u})} du
\end{align*}
The use of the fundamental theorem of calculus is justified since $D$ is monotone and bounded on the interval for sufficiently large $x$, thus Riemann integrable, and since $1/(u\log u)$ is continuous on the interval and thus Riemann integrable, $P'$ is too as the product of two Riemann integrable functions.
Without loss of generality, assume $D(u)$ is eventually monotonically \textit{increasing}. We can bound the integral as:

$$D(x)\underbrace{\int_x^{kx^\lambda} \frac{1}{(u\log u)} du}_{I(x)} \leq \int_x^{kx^\lambda} \frac{D(u)}{(u\log{u})} du \leq D(kx^\lambda)\underbrace{\int_x^{kx^\lambda} \frac{1}{(u\log u)} du}_{I(x)}$$

Finally, we can show that $I(x) \rightarrow \log \lambda$ by computation:
\begin{align*}
& \lim_{x\rightarrow \infty} \int_x^{kx^\lambda} \frac{1}{(u\log u)} du  \\
=& \lim_{x\rightarrow \infty} \log \log \left(kx^\lambda\right) - \log\log x \\
=& \lim_{x\rightarrow \infty} \log\left(\frac{\log kx^\lambda}{\log x}\right) \\
=& \lim_{x\rightarrow \infty} \log\left(\frac{\left(\log k + \lambda\log x\right)}{\log x}\right) \\
=& \lim_{x\rightarrow \infty} \log\left(\frac{1}{\log x}(\log k + \lambda\log x)\right) \\
=& \lim_{x\rightarrow \infty} \log\left(\frac{\log k}{\log x} + \lambda\right) \\
=& \log \lambda = \log(\log_ab) > 0
\end{align*}

\end{proof}

\begin{corollary}[Original ``Meek Models'']

Let $L(C)$ be the validation loss of a model at compute optimality, given by $L(C) = AC^{-\alpha} + E$, where $A,E$ and $\alpha$ are positive constants. Moreover let $g_a, g_h,$ and $g_i$ be annual rates of shared algorithmic progress, shared hardware progress, and compute scaling by a single actor, each greater than $1$. Then the loss difference between the scaling actor and a static actor, given by $L((g_hg_a)^tC_0) - L((g_hg_ag_i)^tC_0)$, goes to zero as $t \rightarrow \infty$.
\end{corollary}
This corollary can be seen immediately by letting $b = g_ig_hg_a > g_hg_a = a$, and $P(C) = -L(C)$. Since $L$ is bounded below, $P$ is bounded above and is therefore a meek metric.

\section{Plot details}
\label{app:plot_details}

We construct smooth trends from constants reported by METR \citep{kwa2026measuring}, starting from the release date of GPT-4–0314. We use their exponential doubling time for the unbounded 50\% time horizon, reported as $\sim$7-month. We also use the logistic-in-task-duration relationship for the probability of task success, which shifts rightward with time according to 50\% time horizon, and estimate probabilities at three fixed human-completion times (4~min, 1~hr, and 4~hr). We anchor the horizon trend to Claude 3.7 Sonnet's reported $\sim$1-hour 50\% horizon (released $1.95$~years after GPT-4 0314), and place the frontier and meek models at the same effective compute at $t=0$ (GPT-4 0314). The success logistic is explicitly reported to have a $5\times$ ratio between the 50\% and 80\% horizons, which locks in the slope of that curve, while the exponential locks in the x-axis shift at a particular point in time. 

For the meek versus mighty estimation, the meek model follows an identical time trend more slowly, such that $\text{meek}(t) = \text{frontier}(\rho t)$. Intuitively, for $b > a > 1$, we want $b^{\rho t}C_0 = a^tC_0$, and therefore $\rho = \log(a)/\log(b) = \log(g_h g_a)/\log(g_h g_a g_i) \approx 0.55$. Here, $g_h \approx 1.36$/yr is the shared hardware-efficiency growth, taken as the FP16 performance-per-dollar trend for top-performing datacenter GPUs without sparsity (a $2.25$-year doubling) from \citet{delsozzo2026nvidia}; we use FP16 as the precision on which modern mixed-precision frontier training runs, and note that the more precision-neutral FP32 trend gives a similar, slightly more conservative rate. The shared algorithmic effective-compute rate is $g_a = 2.8$/yr \citep{ho2024algorithmic}. Frontier training compute has grown at roughly $g_h g_i \approx 4.1$/yr \citep{Sevilla_2022}, implying a frontier-only investment scale-up $g_i \approx 3.0$/yr after dividing out hardware. All curves plot the frontier-minus-meek gap of the titular metric, projected to 2040.

%%%%%%%%%%%%%%%%%%%%%%%%%%%%%%%%%%%%%%%%%%%%%%%%%%%%%%%%%%%%%%%%%%%%%%%%%%%%%%%
%%%%%%%%%%%%%%%%%%%%%%%%%%%%%%%%%%%%%%%%%%%%%%%%%%%%%%%%%%%%%%%%%%%%%%%%%%%%%%%

\begin{lstlisting}[style=pycode]
"""Frontier-vs-meek performance gap on METR time-horizon metrics (constants only)."""

import math
import numpy as np

# ========================= CONSTANTS =========================

# Meekness growth rates per year (hardware, algorithmic, $-investment); set RHO.
HW_DOUBLING_YR = 2.25  # FP16 perf/$, top GPUs, no sparsity -- Del Sozzo et al. (2026)
G_H = 2.0 ** (1.0 / HW_DOUBLING_YR)
G_A, G_I = 2.80, 3.00  # Ho et al. (2024); Sevilla/Epoch

# METR time-horizon trend -- Kwa & West et al. (2025).
DOUBLING_MONTHS = 7.0   # 50% horizon doubling time
RATIO_50_80 = 5.0       # 50% horizon is this many times longer than the 80% horizon
P_HI = 0.80             # the higher success rate anchoring the logistic slope

ANNUAL_GROWTH = 12.0 / DOUBLING_MONTHS 
H_ANCHOR_MIN, T_ANCHOR_YR = 60.0, 1.95   # Claude 3.7 Sonnet: ~1 hr, 1.95 yr after GPT-4-0314

# Starting at years since GPT-4-0314.
REF_YEAR, END_YEAR = 2023.20, 2040.0
TASK_MIN = {"4 minutes": 4.0, "1 hour": 60.0, "4 hours": 240.0}

# ============================== MODEL =======================================

def frontier_50_percent_horizon(t):
    """Frontier 50% time horizon (minutes) at t years since the reference date."""
    return 2.0 ** (ANNUAL_GROWTH * (t - T_ANCHOR_YR) + math.log2(H_ANCHOR_MIN))

def frontier_p_success(t, task_min):
    """Frontier P(success) on a task of length task_min, at year-offset t."""
    logit = lambda x: math.log(x / (1.0 - x))
    # From 50% --> 80% should give you logit(.8) increase, since logit(.5) = 0. 
    log_space_distance_at_t = np.log2(frontier_50_percent_horizon(t)) - np.log2(task_min)
    logit_space_increase_per_log_space = logit(P_HI) / math.log2(RATIO_50_80)
    return 1.0 / (1.0 + np.exp(- logit_space_increase_per_log_space  * log_space_distance_at_t))

def gap(curve, t):
    """Frontier minus meek, where the meek path reaches the frontier at rate RHO."""
    # We want a RHO such that b^(RHO * t)C_0 = a^(t) C_0. That RHO is simply log(a)/log(b), 
    # where b = G_H * G_A * G_I, and a = G_H * G_A
    RHO = math.log(G_H * G_A) / math.log(G_H * G_A * G_I) 
    return curve(t) - curve(RHO * t)

years    = np.linspace(0, END_YEAR - REF_YEAR, 400)
calendar = REF_YEAR + years
gap_horizon = gap(frontier_50_percent_horizon, years)
gap_p       = {lbl: gap(lambda t, T=T: frontier_p_success(t, T), years)
               for lbl, T in TASK_MIN.items()}

# ============================== PLOTTING ====================================

import matplotlib.pyplot as plt
from matplotlib.ticker import FuncFormatter, FixedLocator, MultipleLocator

LABEL_FS, TICK_FS, TITLE_FS, LEGEND_FS = 16, 13, 18, 16

TIME_TICKS = [(1, "1 min"), (60, "1 hr"), (1440, "1 day"), (43200, "1 mo"),
              (525600, "1 yr"), (5256000, "10 yr"), (52560000, "100 yr"),
              (525600000, "1000 yr")]
year_fmt = FuncFormatter(lambda x, _: f"{int(round(x))}")
pct_fmt  = FuncFormatter(lambda v, _: f"{v:.0f}%")

fig, (axL, axR) = plt.subplots(2, 1, figsize=(7.5, 9.5))

axL.plot(calendar, gap_horizon, color="#1f77b4", lw=2.5)
axL.set_yscale("log")
axL.yaxis.set_major_locator(FixedLocator([m for m, _ in TIME_TICKS]))
axL.yaxis.set_major_formatter(FuncFormatter(lambda m, _: dict(TIME_TICKS).get(m, "")))
axL.set_ylim(min(gap_horizon[gap_horizon > 0]), gap_horizon.max() * 1.5)
axL.set_xlabel("Year", fontsize=LABEL_FS)
axL.set_ylabel("Frontier Model vs. Meek Model Gap", fontsize=LABEL_FS)
axL.set_title("METR 50% Time Horizon", fontsize=TITLE_FS, fontweight="bold")
axL.grid(True, which="both", alpha=0.25)

for (lbl, g), color in zip(gap_p.items(), ("#2ca02c", "#d62728", "#9467bd")):
    axR.plot(calendar, 100.0 * g, color=color, lw=2.5, label=lbl)
axR.yaxis.set_major_formatter(pct_fmt)
axR.set_xlabel("Year", fontsize=LABEL_FS)
axR.set_ylabel("Frontier Model vs. Meek Model Gap", fontsize=LABEL_FS)
axR.set_title("METR Projected Success at Fixed Time Horizon", fontsize=TITLE_FS, fontweight="bold")
axR.set_ylim(0, None)
axR.legend(frameon=False, fontsize=LEGEND_FS)
axR.grid(True, alpha=0.25)

for ax in (axL, axR):
    ax.xaxis.set_major_locator(MultipleLocator(4))
    ax.xaxis.set_major_formatter(year_fmt)
    ax.tick_params(axis="both", labelsize=TICK_FS)

fig.tight_layout()
fig.align_ylabels((axL, axR))
fig.savefig("meek_gap_metr.png", dpi=180, bbox_inches="tight")
\end{lstlisting}

\end{document}